\documentclass{article}

     \usepackage[preprint,nonatbib]{neurips_2020}

\usepackage[utf8]{inputenc} % allow utf-8 input
\usepackage[T1]{fontenc}    % use 8-bit T1 fonts
\usepackage{hyperref}       % hyperlinks
\usepackage{url}            % simple URL typesetting

\usepackage{booktabs}       % professional-quality tables
\usepackage{amsthm,amsmath,amsfonts}
\usepackage{mathtools}      % blackboard math symbols
\usepackage{nicefrac}       % compact symbols for 1/2, etc.
\usepackage{microtype}      % microtypography
\usepackage{graphicx}
\usepackage{subfigure}
\usepackage{fourier}
\usepackage[ruled,vlined,linesnumbered]{algorithm2e}
\usepackage[english]{babel}								
\newtheorem{theorem}{Theorem}
\newtheorem{lemma}{Lemma}

\newtheorem{definition}{Definition}[section]

\title{Neighborhood Preserving Kernels for Attributed Graphs}

\author{%
  Asif Salim%\thanks{Use footnote for providing further information
   \\
  Department of  Mathematics,\\
  Indian Institute of Space Science and Technology,\\
  Thiruvananthapuram, India\\
  \texttt{asifsalim.16@res.iist.ac.in} \\
   \And
   Shiju S. S \\
  Architect - AI \& Automation, \\
   Litmus7 Systems Consulting Pvt Ltd, \\
    Cochin, India\\
   \texttt{shijusnair@gmail.com} \\
   \AND
   Sumitra. S \\
  Department of  Mathematics,\\
  Indian Institute of Space Science and Technology,\\
  Thiruvananthapuram, India\\
   \texttt{sumitra@iist.ac.in} \\
}

\begin{document}

\maketitle

\begin{abstract}
We describe the design of a
reproducing kernel suitable for attributed graphs, in which the similarity between the two graphs is defined based on the neighborhood information
of the graph nodes with the aid of a product graph formulation. We represent the proposed kernel as the weighted sum of two other kernels of which
one is an R-convolution kernel that processes the attribute information of the
graph and the other is an optimal assignment kernel that processes label information.
They are formulated in such a way that the edges processed as part of the kernel
computation have the same neighborhood properties and hence the kernel 
proposed  makes a well-defined
correspondence between regions processed in graphs. These concepts are also extended to the case of the shortest paths. We identified the state-of-the-art kernels that can be mapped to such a neighborhood
preserving framework. We found that the kernel value of the argument graphs
in each iteration of the Weisfeiler-Lehman color refinement algorithm can be obtained recursively from the product graph
formulated in our method. By incorporating the proposed kernel on support
vector machines we analyzed the real-world data sets and it has shown superior
performance in comparison with that of the other state-of-the-art graph kernels.
\end{abstract}

\section{Introduction}\label{sec:introduction}

Graph-structured data is used to represent information generated in  fields like bioinformatics, chemoinformatics, social network, natural language processing, computer vision, etc. Hence the development of  efficient methods for the analysis of such graphical data is very essential.  Kernel methods \cite{smola1998learning}, \cite{shawe2004kernel} is  an efficient technique for  analyzing the non-euclidean data. Applying kernel methods on such data requires a kernel function which is formulated using the characteristic properties of the data. In the case of graphs, many approaches have been developed for designing valid kernels. R-convolution kernel \cite{haussler1999convolution} which is a kernel design framework for structured data has been well utilized for designing graph kernels. In this approach, the graphs are decomposed into several parts and separate kernel functions are used to process these parts. Shortest path kernel \cite{borgwardt2005shortest}and  Graph-hopper kernel \cite{feragen2013scalable} etc. are examples of this. Along with this, another approach is to design a specific feature representation process for the graphs in the form of  a vector or to design an explicit mapping of the graphs to a reproducing kernel Hilbert space. For these approaches, graph structures such as random walks \cite{vishwanathan2010graph}, \cite{gartner2003graph}, \cite{kashima2003marginalized}, subtree patterns \cite{ramon2003expressivity}, \cite{shervashidze2011weisfeiler}, \cite{bai2015aligned}, etc. are utilized.

Another design paradigm associated with the graph kernel design is the valid optimal assignment kernel framework \cite{kriege2016valid}. While the R-convolution framework \cite{haussler1999convolution} processes the subgraph structures of the argument graphs one against each other, the assignment kernel framework assigns a bijection between the subgraphs, and kernel processing between them happens in one against one as per assigned bijection. The Weisfeiler-Lehman optimal assignment kernel \cite{kriege2016valid}, optimal assignment kernels on molecular graphs \cite{frohlich2005optimal}, etc. are examples.

Generally, the graph data has labels and attributes associated with each of its   nodes and  edges, where the label is usually a single character while the attribute is a vector.  For example, in the bioinformatics or chemoinformatics applications, the label on the nodes and edges  can be atomic symbol and type of bonds respectively  whereas the attribute information on edges and nodes can be related to that of physical and chemical properties. Thus  the main challenge in designing a  kernel for attributed graphs is  to define a similarity function that could capture all the characteristic properties of data by making use of the label as well as attribute information along with the overall graph structure. By taking into account all these aspects, we have designed a kernel, whose description is presented in this paper.

%Hence for efficiently applying kernel methods on graphs require reproducing kernels that are designed using labels as well as attributes information.

%Most of the existing graph kernels   processes only the  discrete label information in the nodes and edges. Ignoring these continuous data could produce an undesirable effect on prediction tasks. \textbf{However these information can be in the form of continuous attributes as well.  For efficient results using kernel methods on  complex data like graphs require a reproducing kernel that is designed  on the basis of both discrete as well as continuous type of data.}

%The main challenge in designing a  kernel for continuous attributed graphs is  to define a similarity function that could capture all the characteristic properties of data by making use of label as well as attribute information along with the overall graph structure.

%To account for the continuous information, a general approach is to form the graph kernels as summation of node and/or edge kernels. The nodes compared in this manner are characterized based on the graph structures such as random walks, shortest paths etc and using properties such as subgraph isomorphism. But some of the state-of-the-arts designed based on this approach neglects the label information and some are computationally expensive. Even if we consider the labels, certain designs does not have a well defined correspondence between the subgraphs processed and this sometimes can lead to very high similarity scores. Our kernel presented in this paper has been designed to overcome these issues.

Most of the existing graph kernels process only the label information in the nodes and edges ignoring the attributes. The kernel we propose  uses the node and edge labels as well as their attribute information for finding its values.
%For understanding the structural similarity of the argument graphs, a product graph is constructed using   the node/edge label information of neighbors of each nodes through Weisfeiler-Lehman color refinement \cite{weisfeiler1968reduction}. %\textbf{Using the product graph,  we qualitatively define structural similarity of the argument graphs based on the neighborhood information in the nodes and utilizing node and edge labels.} Similarity of the argument graphs are characterized in terms of neighborhood preserving regions. Such a characterization distinguishes between structurally similar and dissimilar edges in the argument graphs and we construct the kernel values by restricting to the similar edges alone.
For that, we formulated the reproducing kernel as  the weighted sum of  two kernels of which one is an R-convolution kernel   that processes  the attribute information   and the  other  is an optimal assignment kernel that processes the label information. 
% Both the kernel processes the structurally similar region alone  and hence the similarity aspect of the kernel is well defined and the process can be faithfully mapped back to the graphs. This helps to define a precise correspondence between the processed regions in the argument graphs unlike most of the state-of-the-arts and to qualitatively define and visualize the concept of similarity between the graphs. Also we derive a characterization via the product graph from which the kernel value can be computed. The R-convolution and optimal assignment components are built upon the edges of the graph. It has to be noted that these concepts can be easily extended to shortest paths. In this regard, we analyzed the R-convolution component with respect to shortest paths as well.
By incorporating the proposed kernels on support vector machines, we experimentally verified its performance using real-world data sets selected from the chemoinformatics domain.

%The neighborhood preserving property we defined can also be treated as a general framework based on the regions in the graphs contributing to the kernel value. In this basis, we also discuss about certain state-of-the-art graph kernels can be mapped to neighborhood preserving framework directly and certain others can be mapped by making appropriate changes.

The paper is organized as follows. The related works are discussed in Section \ref{sec:related}. The notations used are explained in Section \ref{sec:notation}. The graph kernel frameworks used in our kernel design are discussed in Section \ref{sec:frameworks}. We define the concept of neighborhood preserving regions and the neighborhood preserving kernel based on edges in Section \ref{sec:npr}. The extension of these concepts to the shortest paths is discussed in Section \ref{sec:npspk}. Experiments and related discussions are in Section \ref{sec:exp} and the conclusions in Section \ref{sec:con}.

\section{Related works} \label{sec:related}

For processing the attributed graphs, the first attempts were based on random walk kernels. Borgwardt et.al \cite{borgwardt2005protein} improvised the random walk kernels proposed in \cite{gartner2003graph, kashima2003marginalized} to process the attribute information. They used the direct product graph proposed by Gartner et.al \cite{gartner2003graph} to identify similar walks and the kernel is formulated with the help of a modified version of the adjacency matrix of the direct product graph. This modified adjacency matrix encodes the corresponding node and edge kernels in each step of a walk in the direct product graph that corresponds to a similar walk in the argument graphs. The random walks of different lengths are taken into consideration in the higher powers of the adjacency matrix and the final kernel value is calculated using this property. But kernels based on random walks are slower in computation as the length of the walks increase and can suffer from tottering \cite{mahe2004extensions}. Tottering is a phenomenon that results in artificially high similarity values by the repeated visits of the nodes in the graphs.

Zhang et.al \cite{zhang2018retgk} used the return probabilities of random walks of fixed lengths to make node embeddings or a vector representation for the nodes. In this vector, the $i-th$ component of the vector corresponds to the probability of returning to the concerned node with an $i$-step random walk. This vector information is enriched with attribute information. The combined vector is then mapped to a reproducing kernel Hilbert space through kernel mean embedding. The graph kernel is then defined using these mappings.  One shortcoming with kernels based on random walks is that random walks processed may not be structurally equivalent, that is, as far as the regions corresponding to these random walks in the graphs are considered, the resulting subgraphs and hence the nodes correlated need not be structurally equivalent.

Kernel designs based on shortest paths were  used as instances of R-convolution kernels. Shortest path kernel \cite{borgwardt2005shortest} accounts for all shortest paths in the form of a triplet $(u,d(u,v),v)$ where $u,v$ are nodes, and $d(u,v)$ is the shortest path distance between the nodes. All such triplets were convolved against each other in the kernel definition. Note that in this case, only end nodes are considered for processing attribute information. Feragen et.al \cite{feragen2013scalable} extended these concepts to all nodes encountered in the shortest paths and designed GraphHopper kernel that process equal length shortest paths. Note that in these kernels well-defined correspondence between node pairs are absent, i.e, the source node of a shortest path can be paired with either the source or sink node of the counterpart shortest path and hence the same problem with the nodes in between as well since the kernel design does not use labels on nodes and edges. Hence the structural differences characterized by labels are getting neglected.

The subgraph matching kernel proposed by Kriege et.al \cite{kriege2012subgraph} enumerates through subgraphs that preserves subgraph isomorphism. They proposed a product graph definition from which the subgraph isomorphisms in the argument graphs can be identified by enumerating through the cliques in this product graph. The kernel component corresponding to a clique is then defined as the product of the corresponding node and edge kernels and the final graph kernel value is the summation across all such cliques up to a predefined size. Although this kernel maintains structural equivalences, it suffers from a high runtime issue.

The ordered decomposition directed acyclic graph (ODD) kernel \cite{oddkernel} creates a representation of the graph as a bag of well defined directed acyclic graphs (DAG). Each node pair in each DAG pairs are then processed in a similarity function  to derive the kernel value. Note that in this feature extraction procedure, it is difficult to establish a correspondence between the processed regions of the argument graphs.

Graph invariant kernels\cite{orsini2015graph} proposed by Orsini et.al is a framework that incorporates graph invariants with R-convolution design. This kernel is defined in terms of the summation of every possible node kernel between two graphs. Each node pair kernel value is multiplied by a weight as well. The weight is determined by the structural importance of that node pair in a particular R-convolution decompositions of the graphs. The weight function also incorporates a similarity value characterized by node invariants such as Weisfeiler-Lehman labels, spectral colors, etc. The maintenance of structural correspondences in this approach depends upon the selection of node invariants and it can have higher runtime problems depending upon the R-convolution relation design.

The approaches discussed above are based on certain graph substructures and each having a specific feature extraction process from them. Another fundamental approach applied to attributed graphs is based on the discretization of attributes. Propagation kernel \cite{neumann2012efficient} by Neumann et.al, Hash graph kernels by Morris et.al \cite{morris2016} etc. are examples. In Propagation kernels,  the attributes are discretized using a hashing mechanism. The information in the graphs is propagated at each iteration across the nodes. Each iteration gives a unique instance of the graph that carries structural information and the final kernel value is defined as the summation of individual kernels at these iterations. For efficient computation, it makes use of locality sensitive hashing. It is also helpful in handling missing attributes as well. In Hash graph kernels, the idea is to convert the attributes to discrete ones using well-defined hashing functions. The hashing mechanism is done in several iterations and kernels values at these iterations are summed up. Although both of the above approaches are scalable, it is coming at the cost of transformation of the original attributes.

\section{Notation}\label{sec:notation}

We define an undirected graph $G$ as $(V,E)$, where $V$ is the set of nodes and  $E = \{(u,v): \mbox{u  is incident on v} \}$ is the set of edges. Define a mapping $l: V \cup E \to \Sigma$, such that $l(v)$   is the label associated with a node $v$, $l((u,v))$ is the label associated with an edge $(u,v)$, and  $(\Sigma, \leq)$ is a total ordered set  that consists of node and edge labels.

%\textbf{To account for the difference in cardinality of set of vertices and edges between two graphs $G=(V,E)$ and $G'=(V',E')$, we define $\tilde{V} = V \cup d_v$, $\tilde{E} = E \cup d_e$ and $\tilde{V}' = V' \cup d_v$,  $\tilde{E}' = E' \cup d_e$ where $d_v$ is a finite collection of dummy nodes and $d_e$ is a finite collection of dummy edges to solve for cardinality difference.}
The  shortest path between the node $u_1$ and $u_n$ is represented as $\Pi(u_1,u_n)$ and its length as $|\Pi(u_1,u_n)|$ where,   the length of the path   is defined as number of edges involved. The graph density is defined as the ratio of number of edges in the graph to the maximum possible edges where parallel edges are not assumed. $\mathcal{X}$  is assumed to be a separable metric space consisting of discrete structures like string, trees or graphs and  $(\Sigma_C, \leq)$ be a total ordered set containing Weisfeiler-Lehman (WL) refined labels.
% and $l_C(v)$ is the \textbf{and let these WL refined labels  be generated from mapping $l_C:V \rightarrow \Sigma_C.$} We define $h$ as the number of times WL color refinement is done.

\section{Methods}\label{sec:frameworks}
The discussion about the kernel frameworks we used in our design is given below.

%\textbf{We use the R-convolution framework and valid optimal assignment kernel in our kernel design. The final kernel is proposed as a linear combination of these two components. The R-convolution component is designed to process  attributes. Valid optimal assignment kernel framework is used to process the labels alone. Note that these frameworks are proposed to design kernels over general structured data.}

\subsection{R-convolution kernels}
The R-convolution kernel formulated  by Haussler \cite{haussler1999convolution} is a generalized framework for processing structured data. This framework involves the decomposition of data into its constituent parts and a kernel is defined in terms of such decomposition.   The shortest path kernel \cite{borgwardt2005shortest}, subgraph matching kernel \cite{kriege2012subgraph}, etc. are examples of R-convolution kernels.

Let $x \in \mathcal{X}$. Assume that $x$ can be decomposed to $D$ number of components or parts. Let   $\tilde{x} =\{ x_1,\dots,x_D \} $ be this decomposition, where   $x_i \in X_i$ and  $X_i, 1 \leq i \leq D$, be a  non empty, separable metric space.  Define a relation $R$ as $$ R = \{(\tilde{x},x) \in X_1 \times  \dots \times X_D \times \mathcal{X} \vert \tilde{x} \text{ are parts of } x\} $$ That is, $R$ is true iff $\tilde{x}$ is a valid decomposition of $x$.
Now it is possible to define the inverse relation, $R^{-1}(x) = \{\tilde{x} : R(\tilde{x},x)=1\}$. Note that $R$ is finite if $R^{-1}$ is finite $\forall x \in \mathcal{X}$. Then  the R-convolution kernel $K_{R_{conv}}: \mathcal{X} \times \mathcal{X} \to \mathbb{R}$ is defined as $$ K_{R_{conv}}(x,y) = \sum_{\tilde{x} \in R^{-1}(x)}  \sum_{\tilde{y} \in R^{-1}(y)} \prod_{i=1}^{D} k_i(x_i, y_i)$$ where $(x_i, y_i)$ is  the $i^{th}$ component of  $(\tilde{x}, \tilde{y})$  and  $k_i(.,.):  X_i \times X_i \to \mathbb{R},  1 \leq i \leq D$ is the kernel corresponding to $i^{th}$ component.

\subsection{Valid optimal assignment kernel framework}
The optimal assignment kernel is defined as follows: Let $\mathcal{X}^n, n \in \mathbb{N} $ denote the set of all $n$-element
subsets of $\mathcal{X}$ and $B(X, Y )$, the set of all bijections between $X, Y$, where   $X, Y \in \mathcal{X}^n$. The optimal assignment kernel $K_B^k$:  $\mathcal{X}^n \times \mathcal{X}^n \to \mathbb{R}$  is defined as
\begin{equation}\label{base}
K_B^k (X,Y) = \underset{\beta \in B(X,Y)}{max}\; W(\beta) \;\;\;\; where \;\;\;\; W(\beta) = \sum_{(x,\beta(x))} k(x,\beta(x)),
\end{equation}
where $k$ is a strong  kernel defined on $\mathcal{X} \times \mathcal{X}$ \cite{kriege2016valid}.
The above concept can be extended to the case where the underlying domain consists of sets with unequal cardinality by adding dummy objects $d$ to the smaller set, where $k(d,x)=0, \; \forall x \in \mathcal{X}$. The concept of strong kernel is discussed below.

%In the case of graphs, $\mathcal{X}$ can be the set of nodes or edges of the graph.  In the case where $X$ and $Y$ are having different cardinality, the smaller set can be filled with dummy objects $d$ such that $k(d,x)=0, \; \forall x \in \mathcal{X}$. \textbf{Note that addition of dummy objects does not cause a change in the value $W(\beta)$.}
%	The above equation considers all possible bijections and assigns that bijection which maximizes the kernel value $K$.}

%Kriege et.al \cite{kriege2016valid} showed that a function defined of the form \ref{base} is an optimal assignment kernel  iff $k$ is strong. Kriege et.al \cite{kriege2016valid} characterize these base kernels as strong kernels. The notion of "strong" kernels are defined with respect to an inequality which aids in establishing optimal bijection. \textbf{They also proved that this strong kernel can be defined with respect to a hierarchy and it can be computed in linear time given this hierarchy.} The hierarchy is supposed to create a hierarchical partitioning of the domain corresponding to the optimal assignment kernel.

\subsubsection{Strong kernels and hierarchies} \label{strong}

The strong kernel can be explained using two different concepts:

1. In terms of an inequality constraint on the kernel values.

A function $k : \mathcal{X} \times \mathcal{X} \rightarrow \mathbb{R}^+$ is called strong kernel if $k(x,y) \geq \; min \; \{k(x,z), k(z,y)\} \;\;\; \forall x, y, z \in \mathcal{X}$. That is, once we consider $x$ and $y$, there is no other element $z$ in $\mathcal{X}$ where both $x$ and $y$ are more similar to $z$ than themselves.
%Intuitively, this inquality constraint of strong kernel ensures the assignments done between parts are optimal and keeps up with Equation \ref{base} and it is under the assumption that a data point is more similar to itself.}

2. In terms of hierarchy defined on the domain of the kernel.

The hierarchy can be constructed  in the form of a rooted tree, $T$ as follows. We assume that the leaves of $T$ corresponds to elements in $\mathcal{X}$. It has to be noted that tree forms a nested structure, that is, each inner node in $T$ apart from the leaves corresponds to a subset of elements in $\mathcal{X}$. For example, in an example hierarchy in Figure \ref{heir}, the node $v$ is an inner vertex that corresponds to nodes $a$, $b$ in $\mathcal{X}$ and node $r$ corresponds to node $c$ in $\mathcal{X}$.

A weight is defined to each of the nodes in $T$. Let $w : V (T) \to \mathbb{R}^+$ be a weight function such that $w(v) \geq w(p(v))$ for all $v$
in $T$ where $p(v)$ denotes parent node of the node $v$, and $V(T)$ is the set of nodes in $T$. The tuple $(T, w)$ is referred as hierarchy. It has to be noted that base kernels which are strong in $\mathcal{X} \times  \mathcal{X} $mentioned earlier can be derived from the hierarchy structure \cite{kriege2016valid}.

\begin{figure}[]
	\centering
	\includegraphics[scale=0.4]{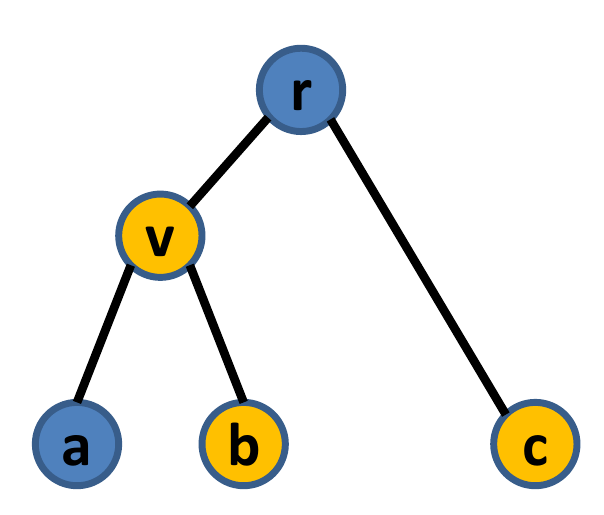}
	\caption{Example for hierarchy}
	\label{heir}
\end{figure}

With the above concepts, Kriege et.al\cite{kriege2016valid} proved that \textit{a kernel $k$ on $\mathcal{X} \times  \mathcal{X}$ is strong if and only if it is induced by a hierarchy on $\mathcal{X}$}.

It has to be noted that the hierarchy corresponding to a base kernel $k$ is not unique. Considering this, to prove that a kernel is an optimal assignment on a structured data, it is enough to prove that the underlying base kernel is strong
and there exists a hierarchy which induces the corresponding  base kernel from which equation \ref{base} can be calculated.
%\end{enumerate}

Examples for optimal assignment kernels are kernels defined on vertices and edges, WL-optimal assignment kernel by Kriege et.al \cite{kriege2016valid}, and optimal assignment kernels on molecular graphs by Frohlich et.al \cite{frohlich2005optimal}. %Note that the work done by Frohlich et.al \cite{frohlich2005optimal} was mapped later to optimal assignment framework by Kriege et.al \cite{kriege2016valid}.

%Now we define the neighborhood preserving concept. The idea is to identify structural similar regions of two argument graphs on the basis of how similar the neighborhood of each node is. For this we make use of labels of the nodes and that of its neighbors. Then we identify the regions in the graphs that have similar neighborhood. The entire process is explained in the following section.

\section{Neighborhood Preserving Kernel}\label{sec:npr}
In this section  we discuss the design aspects of the neighborhood preserving kernel we formulated.

\subsection{Neighborhood preserving regions}

The neighborhood preserving (NP) kernels process the structurally similar regions of the argument graphs. In our design, the concept of structural similarity is defined using neighborhood preserving property, whose concept is given below.

%The idea behind Neighborhood preserving (NP) kernels is to qualitatively define the similarity and dissimilarity of two given graphs and to process the similar regions for kernel construction. The distinction between similar/dissimilar regions is differentiated on the basis of neighborhood of the nodes in the graphs modeled through the neighborhood preserving property. We used product graph formulation to identify the neighborhood preserving regions of the graph. \textbf{First we use a product graph formulation to identify the neighborhood preserving regions of the graphs and hence differentiating structural similar/dissimilar regions.}

To formalize the neighborhood preserving regions, 1-dimensional WL color refinement \cite{weisfeiler1968reduction} is done on the graphs.  The WL color refinement algorithm creates a representation for each node in the graph in the form of a string. The node label of the node is the first element of the string and the remaining elements are node labels of the neighboring nodes in lexicographic order. This string representation is assigned a new label through a hashing function and the above process is repeated with the new assigned label. If the graphs are devoid of any node labels, the same label for all nodes can be assumed using some strategy: for example, the node label can be taken as  the degree of the node  in the first iteration. A product graph \cite{gartner2003graph} is then constructed to find the structural similar regions, where the product graph is defined as follows :

%\textbf{With the new colors or labels in the nodes direct product graph proposed by Gartner et.al  is done.} The process is detailed as below.
%and $L_{\mathcal{N}(v)}$ be the list which is obtained by mapping individual nodes in $\mathcal{N}(v)$ by the labeling function $L(v_i)$ where $v_i \in \mathcal{N}(v)$ and elements in the list are following the order defined by $\mathcal{T}$ .

\begin{definition}{[\textbf{Direct product graph}]}:
	Let $G=(V,E)$ and $G'=(V',E')$ be two graphs under consideration. The direct product graph, $G_P = G \times G'$, represented as $(V_{P},E_{P})$ is defined as
	\begin{equation*}
	V_{P} = \{(u,u') \in V \times V' : l_C(u)=l_C(u')   \}
	\end{equation*}
	\begin{equation*}
	\begin{split}
	E_{P}
	=\big\{ &\big((u,u'),(v,v')\big) : (u,v) \in E \wedge (u',v') \in E' \wedge \\
	& l((u,v)) \; =\; l((u',v')) \big\}
	\end{split}
	\end{equation*}
\end{definition}
An example is shown in  figure \ref{fig1}(a),(b),(c) of two sample graphs, their WL coloring, and product graph formation. From the product graph the set of neighborhood preserving or structurally  similar edges of the two given graphs based on preserving its structure and node/edge labels information can be deducted as per the definition given below.

\begin{figure*}[h]
	\centering
	\includegraphics[scale=0.3]{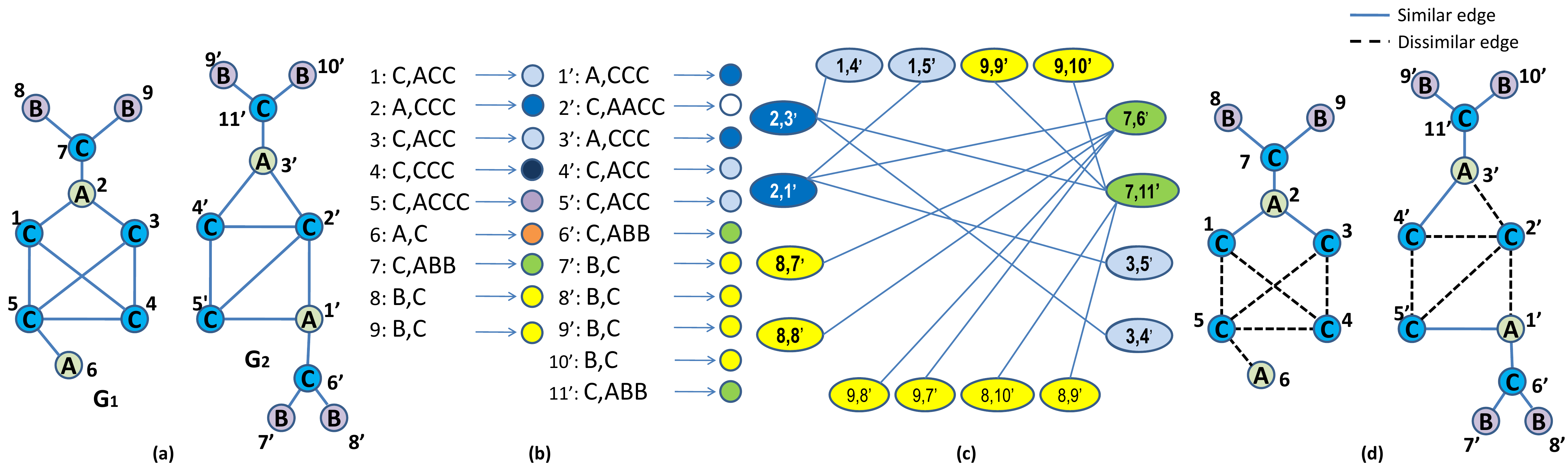}
	%\fbox{\rule[-.5cm]{0cm}{4cm} \rule[-.5cm]{4cm}{0cm}}
	\caption{(a) Two sample graphs $G_1$ and $G_2$, (b) WL color refinement, (c) direct product graph, (d) separation of neighborhood preserving or structurally similar edges(bold lines) and dissimilar edges(dashed lines) }
	\label{fig1}
\end{figure*}

\begin{definition}{[\textbf{Neighborhood preserving edges}]}:
	
	The 	set of neighborhood preserving edges in graph $G$ with respect to $G',$ is $$\mathcal{S}_{G:G'} = \{(u, v) \in E : ((u, u'),(v, v'))  \in E_{P}   \}$$
	
	%and that of
	%$G'$ with respect to $G,$ is $$\mathcal{S}_{G':G}  = \{(u', v') \in E' : ((u, u'),(v, v'))  \in E_{P}  \}.$$
\end{definition}
%\textbf{[This second part is required because the NP edges cannot be defined with a single graph alone. it is always defined with respect to the other graph]}
%\begin{definition}{[\textbf{Structurally dissimilar edges }]}:
%	Those edges in $G_1$ and $G_2$ which does not come under the above definition forms the dissimilar edges of their respective graphs.
%\end{definition}

%	\begin{definition}{[\textbf{Neighborhood preserving regions }]}: \label{npr}
%		The connected components formed by the structurally similar edges form the neighborhood preserving regions of the respective graphs.
%	\end{definition}	
An example of the extraction of neighborhood preserving edges of the graphs is shown in figure \ref{fig1} (d). Note that the WL color refinement algorithm basically processes the information about neighborhood of the nodes. Now, corresponding to each $(u,v) \in \mathcal{S}_{G:G'} $  there exists atleast one $ (u',v') \in E' $ such that $u, u'$ as well as $v, v'$ has the same label and neighborhood information.  Hence the edges are termed as \textit{neighborhood preserving edges.} This helps to process kernel computation with subgraphs having well-defined correspondences in terms of the neighborhood of the component nodes and it also gives a visualization of the structural similarity between the graphs.

%So if we take a structurally similar edge in a graph there exist a corresponding edge in the counterpart graph where we can establish a correspondence between the nodes of these edges in terms of identical neighborhood.

We designed a kernel that defines the similarity of the argument graphs with the aid of  neighborhood preserving edges, whose description is given below.

%We argue that in graph kernel designs, we can get a well defined and accurate similarity measures if we process on the neighborhood preserving edges of the graphs.  It can be seen that if we design a kernel that processes the edges characterized by nodes having common WL iterated labels in both of the graphs, we can limit the feature extraction process to be confined with neighborhood preserving edges only. With this strategy we define our kernel as a combination of R-convolution and optimal assignment design methodologies as described in the next section.

%\section{Neighborhood preserving kernel} \label{sec:npk}
%The neighborhood preserving (NP) kernel is defined as weighted sum of two kernels, i) neighborhood preserving edge kernel which is an R-convolution kernel and ii)neighborhood preserving optimal edge assignment kernel as described in the following sections.
\subsection{ Neighborhood preserving edge kernel} \label{sec:npek}

%	\textbf{what about argument on edges}
Let  $\kappa_V$ and $\kappa_E$ be  positive semi definite kernels defined on $V \times V'$  and $E \times E'$ respectively and $\lambda \in \mathbb{R}^+$ be a weight function. Then the neighborhood preserving edge kernel, $K_{NPE}$, is defined as

%Then neighborhood preserving (NP) kernel is defined as
%\begin{equation} \label{main:def}
%K(G,G')= \alpha K_{NPE}(G,G') + (1-\alpha) K_{NPO}(G,G')
%\end{equation}
%and  $\alpha$ be a tuning parameter in the range (0,1).

\begin{equation}\label{Knpe}
\begin{split}
K_{NPE} (G,G')&= \lambda(G,G').\\
&  \sum_{(u,v) \in E} \sum_{(u',v') \in E'}  \kappa_V(u,u'). \kappa_E \big((u,v),(u',v')\big).  \kappa_V(v,v')
\end{split}
\end{equation}
where $\kappa_V$ is defined in terms of the attribute information of argument nodes and $\kappa_E$ is defined as
\begin{equation*}
\begin{split}
\kappa_E \big((u,v),(u',v')\big)
&=\delta\big( l(u,v), l(u',v')\big). \\
&\delta (l_C(u),l_C(u')) . k(\big((u,v),(u',v')\big)). \delta (l_C(v),l_C(v'))
\end{split}
\end{equation*}
where $\lambda$ is a normalization constant,  $\delta$ is  the Kronecker delta function and $k$ is a valid kernel defined in terms of attributes on edges. If no attributes are present on edges, the value of  $k$ can be taken as 1. Note that delta functions ensure that the edges processed for the kernel computation are from neighborhood preserving edges only and hence the name for the kernel. The function of weight $\lambda$ is to give a normalizing effect to avoid artificially high similarity value if the argument graph sizes vary unevenly and if some  WL refined labels in a graph are very high compared to its counterpart resulting in a large number of edges of a particular type.

\begin{theorem}\label{theorem1}
	The neighborhood preserving edge kernel is a valid kernel.
\end{theorem}
\begin{proof}

	We prove $K_{NPE}$  as a R-convolution kernel \cite{haussler1999convolution}. Let $\Sigma^3$  the set of   strings of length three formed by members of $\Sigma_C$ and $\Sigma$, that is,   $\Sigma^3=\{xyz \vert x, z \in \Sigma_C , x \leq z, y \in \Sigma \}$.  Consider two graphs $G$ and $G'$ and $\Lambda \subseteq \Sigma^3$ where   $$ \Lambda= \underset{(u,v) \in E}{\cup } l_C(u)\odot l(u,v) \odot l_C(v) \;\;\bigcap\;\; \underset{(u',v') \in E'}{\cup } l_C(u') \odot  l(u',v') \odot l_C(v')  $$ where $\odot$ is a concatenation operator, $l_C(u) \odot  l(u,v) \odot l_C(v), \;l_C(u') \odot  l(u',v') \odot l_C(v') \in \Sigma^3$ .  We call elements of $\Lambda$ as \textit{WL refined edge address}.
	
	%are the WL refined colors of nodes $u,v$ that forms $(u,v)$ edge in $G$ and nodes $u',v'$ that forms $(u',v')$ edge in $G'$ respectively.
	
	Now we define a relation $R$ corresponding to each member  $ \epsilon \in \Lambda.$ Let $R(e,G^\dagger , G; \epsilon)$ be a relation where $R(e,G^\dagger, G; \epsilon) =1$ iff $G^\dagger$ is the graph obtained from $G$ if an edge $e$ corresponding to a member $\epsilon$ in $\Lambda$ is removed. Now $R^{-1}(G;\epsilon)=\big\{(e,G^\dagger) : R(e,G^\dagger, G; \epsilon) =1\big\}$ is the decomposition of a graph into an edge with \textit{WL refined edge address} $\epsilon$ and rest of the graph. Now  $K_{NPE}$ is defined as
	%the sum of certain R-convolution kernels corresponding to each $\epsilon \in \Lambda$ as follows,
	\begin{equation}\label{Rconv}
	K_{NPE}(G,G') = \sum_{\epsilon \in \Lambda}\; \frac{1}{|E_\epsilon \times E_\epsilon'|} \sum_{\underset{(e',G^{\dagger'}) \in R^{-1}(G';\epsilon)}{\underset{(e,G^\dagger) \in R^{-1}(G;\epsilon)}{}}} k_{edge}(e,e') \times k_{triv}(G^{\dagger},G^{\dagger'})
	\end{equation}
	where $E_\epsilon,E_\epsilon'$ are set of edges in $G,G'$ corresponding to  \textit{WL refined edge address} $\epsilon$, $E_\epsilon \times E_\epsilon'$ being their cartesian product, $k_{triv}$ is a trivial kernel whose value is always 1 and \begin{equation}\label{Kedge}
	k_{edge}(e,e') = \kappa_V(u,u') \times k(e,e') \times \kappa_V(v,v')
	\end{equation}
	where $e=(u,v)$ and $e'=(u',v')$. Now $k_{edge}$ is  a valid kernel \cite{cortes2004rational}. Also, $|E_\epsilon \times  E_\epsilon'|$ is a well defined function and  hence $K_{NPE}$ is a R convolution kernel.
\end{proof}
The effect of the cardinality of $\Lambda$ has to be considered and hence included   the weight $\frac{1}{|E_\epsilon \times  E_\epsilon'|}$ in equation \eqref{Rconv}. The  $\lambda$ in equation \eqref{Knpe}  can be considered as  the counterpart of this term.

%This weight multiplication maintains positive definiteness once it can be proven  as a well defined function on $G,G'$ \cite{shawe2004kernel}.  Note that $|E_\epsilon \times  E_\epsilon'| = |E_\epsilon| \times  |E_\epsilon'|$ since it counts the number of decompositions as defined in the relation $R(e,G^\dagger , G; \epsilon)$.  Since it is a well defined function on $E_\epsilon$ and  $E_\epsilon'$ and therefore in $G$ and $G'$, the R-convolution kernel in equation \ref{Rconv} is positive semi definite. Hence $K_{NPE}$ is positive semi definite.

By constructing a product graph, the information used in  $K_{NPE}$ kernel can be extracted as explained below.

%\textbf{Note that the $K_{NPE}$ kernel can be formulated for further iterations of WL color refinement. In this case the final kernel value is taken as the summation of kernels in individual iterations.}

\begin{theorem}\label{theorem:2}
	Each edge in the direct product graph corresponds to a product calculation in convolution operation defined in equation \eqref{Rconv} except for a few edges of the form $\big((u,u'),(v,v')\big) \in E_P$ where $l_C(u) = l_C(v) = l_C(u') = l_C(v')$ and $l(u,v) = l(u',v')$.
\end{theorem}
\begin{proof}
	\begin{figure}[h]
		\centering
		\includegraphics[scale=0.35]{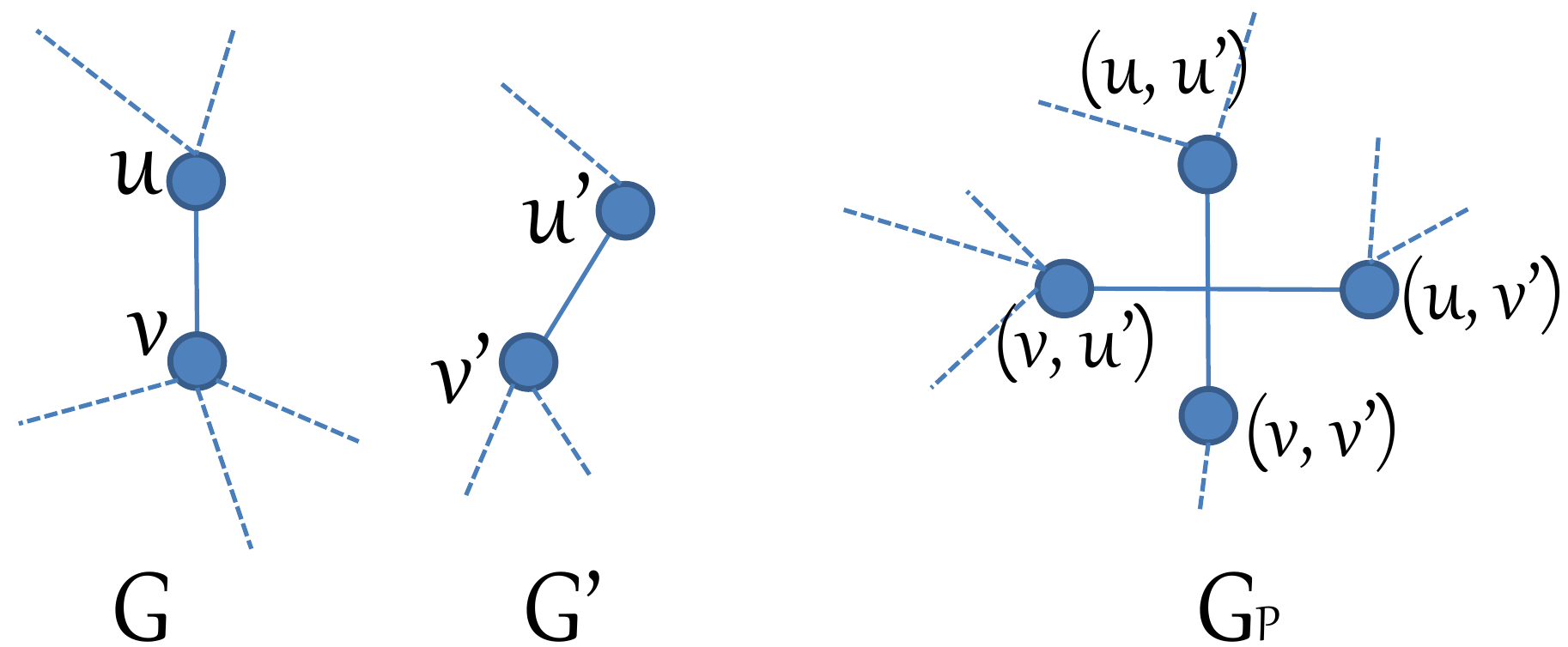}
		%\fbox{\rule[-.5cm]{0cm}{4cm} \rule[-.5cm]{4cm}{0cm}}
		\caption{Two nodes $u, v$ in the graph $G$ and two nodes $u', v'$ in the graph $G'$ with the same WL refined labels and corresponding edges having a similar label. The product graph $G_P$ containing an additional pair of nodes and an additional edge.   }
		\label{fig2}
	\end{figure}
	Consider a particular edge $(u,v) \in E$  in a graph $G$ which is neighborhood preserving. Note that the nodes in the direct product graph contain all combination of nodes from graphs $G$ and $G'$ whose label and neighborhood information is identical. Hence the edge $(u,v)$ is a  part of set of edges of the form $\big((u,u'),(v,v')\big)$ in direct product graph  where $u' \in \{x \in V': l_C(x) = l_C(u)\}$, $v' \in \{y \in V': l_C(y) = l_C(v)\}$, $(u',v') \in E'$ and $l(u,v) = l(u',v')$. Hence these set of edges correspond to a convolution of edge $(u,v) \in E$ with its counterpart edges in $E'$. Similarly we can find out other convolutions corresponding to every edge in $G$.
	
	But a duplication arises in a special case when $l_C(u) = l_C(u') = l_C(v) = l_C(v')$ and $l(u,v) = l(u',v')$. Consider such a case as shown in the figure \ref{fig2}. We can see that in the direct product graph, $G_P$, four nodes are created with two edges. But as per convolution operation, the edge pair $(u,v)$ and $(u',v')$ shall be counted only once. Hence avoiding this duplication, $K_{NPE}$ can be calculated from $G_P$.

\end{proof}

The NPE kernel can be found out for multiple iterations of WL color refinement. In this case the final kernel is taken as the sum of kernels in the individual iteration. 

The neighborhood preserving edge kernel processes  node and edge attributes with the labels on them acting as a guidance for the R-convolution process.  For processing the node and edge labels alone, another kernel named neighborhood preserving optimal edge assignment kernel is formulated, whose description is given in the next  section. The objective is to analyze the importance of the label information.

%It also helps to check the effectiveness of the final proposed kernel which is the linear combination of the two kernels.

\subsection{Neighborhood preserving optimal edge assignment kernel}

The neighborhood preserving optimal edge assignment kernel $K_{NPO}$ processes only node labels and  defined as

%\textbf{Although edges have been utilized to construct the $K_{NPE}$ kernel, it is the node and edge attribute information that is computed.} Hence we design neighborhood preserving optimal edge assignment kernel, $K_{NPO}$,  that  makes use of the label information independently, where 	

\begin{equation}\label{koea}
K_{NPO} (G,G')= 	\sum_{\epsilon \in \Lambda}\; \text{minimum}(|E_\epsilon|, |E_\epsilon'|)
\end{equation}
The minimum value is taken as it represents a normalized score of structural similarity. Note that for $K_{NPO}$ computation also, the edges processed are only neighborhood preserving ones. Note that $K_{NPO}$ can also be defined for multiple WL color refinement steps where kernel value in individual iteration is summed up. 

With the aid of an optimal assignment kernel $K_B(E,E')$, we prove  $K_{NPO} (G,G')$ as a valid kernel. We define $K_B(E,E')$ as,

%Now to prove  $K_{NPO} (G,G')$ as a valid kernel,  we define it as a valid optimal assignment kernel, $K_B(E,E')$ , between the sets of edges in $G$ and $G'$  as follows.

%We prove the positive semi definiteness of $K_{NPO}$ with the following lemmas.
%This kernel is psd since it is an optimal assignment kernel \cite{kriege2016valid} with the base kernel \cite{kriege2016valid} being delta kernel.

%	Note that the convolution operation is happening only over the edges having identical \textit{WL refined edge addresses}. This ensures that the kernel values is getting contributed from neighborhood preserving regions only. Also note that the weight $\frac{1}{|E_\epsilon \times E_\epsilon'|}$ in Equation \ref{Rconv} ensures a normalization as it avoids getting artificially high kernel values if edges corresponding to some   \textit{WL refined edge address} in a graph is very high compared to its counterpart.

%With  the definition of $K_{NPO}$,  information is also embedded as a seperate similarity measure.

%Now to prove  $K_{NPO}$ as a valid kernel,  we define it as a valid optimal assignment kernel $K_B$ formulated between the set of edges $E$ and $E'$ of graphs $G=(V,E)$ and $G'=(V',E')$ respectively. For that a bijection has to be formed between the members of $E$ and $E'$. Here we assume that the sets among $E$ and $E'$ that contain least cardinality have enough dummy members to make up in difference of cardinality.

%\textbf{Define $K_B: (E \cup E') \times (E \cup E') \rightarrow \mathbb{R^+}$ as}
\begin{equation}\label{main}
K_{B}(E,E') = \underset{\beta \in B(E,E')}{max} \;\;\; \underset{\big((u,v),\beta(u,v)\big)}{\sum} \;\; \sum_{i=1}^{h} k_b^i\big( (u,v), \beta(u,v) \big)
\end{equation}
where $B(E,E')$ is the set of all bijections between the members of the set $E$ and $E'$ corresponding to the base kernel $k_b^i$ at iteration:$i$ of WL color refinement, $(u,v)\in E$, and $\beta(u,v) \in E'.$ Here we assume that the sets among $E$ and $E'$ that contain the least cardinality have enough dummy members to make up for the difference between their cardinality and WL iteration is done $h$ times.

The base kernel $k_b^i$ is defined as,

\begin{equation} \label{base kernel def}
k_b^i\big( (u,v), (u',v') \big) =\left\{
\begin{array}{ll}
\begin{split}
1,&\; if \;l_C^i(u) = l_C^i(u') \wedge l_C^i(v) = l_C^i(v') \\& \wedge
l\big( (u,v)\big) = l\big( (u',v')\big)
\end{split}
\\ 0, \; otherwise  \\
\end{array}
\right.
\end{equation}
where $l_C^i(.)$ indicates WL color of the concerned node at iteration:$i$ of the WL color refinement. The base kernel defined above is strong as the cardinality of its range set is two. \cite{kriege2016valid}.
%\begin{lemma}\label{lemma:1}
%	The base kernel defined in \ref{base kernel def} is strong.
%\end{lemma}

%\begin{proof}

%\end{proof}
%||%\begin{lemma}\label{lemma:2}
%The base kernel $k_b$ is induced by an hierarchy.
%\end{lemma}

%\begin{proof}

%For building the hierarchy $T$, a root node is fixed for each of the \textit{WL refined edge address} that occur for the entire graphs in the dataset under consideration. The child nodes for these root nodes are built for the next iteration of WL color refinement step. For this next iteration, note that the \textit{WL refined edge addresses} and hence new nodes occur always correspond to a \textit{WL refined edge address} in the previous level of the hierarchy. Thus it is possible to define a parnt-child relation in the form of an edge in $T$. In this way the newer levels of hierarchy can be formed for the proceeding iterations of WL color refinement.

The hierarchy tree corresponding to base kernel $k_b$ can be constructed in the following way. Consider the application of the WL color refinement algorithm on the given graphs. Corresponding to each  WL refined edge address   obtained in the first iteration, a root node is created. The elements of WL refined edge addresses obtained for iteration $i > 1$, forms the nodes for the level $i$ of the hierarchy tree. As the nodes in the higher level are generated from the nodes in the lower level, a parent-child relationship between two subsequent levels can be established. Thus a tree structure $T$  can be built by adding edges between parent and children.

\begin{lemma} \label{lemma:3}
	$K_{B}(E,E')$ can be calculated from the hierarchy.
\end{lemma}

\begin{proof}
	With the help of the hierarchy $T$, $K_{B}$ can be calculated as a histogram intersection kernel \cite{kriege2016valid} as follows.  Corresponding to each graph $G$, a histogram vector $G_V$ of length $V(T)$ is created, where $V(T)$ is the set of nodes of the tree. The $i^{th}$ element of $G_V$  is the number of times it produces the $i^{th}$ element  of $V(T)$ during the WL color refinement algorithm procedure. The histogram vector calculation is explained  in  Algorithm \ref{alg3}.
	
	%For this, for each graph, a histogram vector $(G_V)$ of dimension $|T(V)| \times 1$ is found out where $T(V)$ is the set of nodes in the hierarchy $T$. That is, for each graph under consideration, an associated weight corresponding to each node of the hierarchy is found out. The weight counts the number of occurrence of \textit{WL refined edge addresses} corresponding to that node in $T$ for each graph.

	\begin{algorithm}
		\label{alg3}
		\DontPrintSemicolon
		\SetAlgoLined
		%\KwResult{Write here the result}
		\SetKwInOut{Input}{Input}\SetKwInOut{Output}{Output}
		\Input{The hierarchy $T$ for $h$ iterations of WL color refinement algorithm and a graph $G=(V,E)$.}
		\Output{ Histogram vector $G_V$}
		\BlankLine
		Assign a unique location starting from 1 to $|V(T)|$ to all the nodes in $T$ \;
		Initialize $G_V$ as $|V(T)| \times 1$ vector  \;
		\For{every edge $e$ in $E$}{
			Find the leaf node $n$ in $T$ corresponding to \textit{WL refined edge addresses} of edge $e$ at WL iteration: $h$\;
			$i = h$\;
			\While{$i \geq 1$}{
				$G_V($location$(n)) = G_V($location$(n)) + 1$\;
				$n$ = parent($n$) \;
				$i = i-1$ \;	
			}	
		}
		
		\caption{Computation of histogram vector $G_V$}
	\end{algorithm}
	
	Note that the above algorithm corresponds to a tree traversal for each edge in a graph starting from the leaf nodes. Since the hierarchy contains all possible occurrences of \textit{WL refined edge addresses} in $h$ iterations of WL color refinement, each visit of a node in $T$ at level $i$ with respect to the edge under consideration correspond to an occurrence of that particular \textit{WL refined edge address} characterizing the node at iteration:$i$. Hence the tree traversals as per the algorithm give the number of occurrences of \textit{WL refined edge addresses} in the entire $h$ number of WL iterations, that is,  $G_V$. Let $G_{V}$ and $G_{V}'$ be the histogram vectors of graphs $G$ and $G'$. Then $K_B(E,E')$ in equation \eqref{main} can be calculated as histogram intersection kernel of $G_{V}$ and $G_{V}'$ \cite{kriege2016valid}, that is,
	\begin{equation}
	K_{B}(E,E') = \sum_{i=1}^{|G_V|} \text{minimum} \left(G_V(i), G_V'(i)\right)
	\end{equation}
	
	It is straightforward now to establish that $K_{B}(E,E')$ for $h$ iterations of WL color refinement steps is the summation of $K_{NPO}$ kernels in individual refinement steps. That is, $$ 	K_{B}(E,E') = \sum_{i=1}^h K_{NPO}^i (G,G')$$ where $K_{NPO}^i$ is the NPO kernel at iteration $i$.
\end{proof}
On the basis of above discussion, it is clear that neighborhood preserving optimal edge assignment kernel is a valid optimal assignment kernel.

It has to be noted that $K_{NPO}$ in actual experiments is calculated as explained in Section \ref{compute methods}.

\subsection{Neighborhood preserving kernel definition}

The neighborhood preserving (NP) kernel for $h$ iterations of WL color refinement algorithm is defined as
%the weighted sum of sum of $K_{NPE}$ and sum of $K_{NPO}$ in the individual iterations.
\begin{equation} \label{main:def}
K(G,G')= \alpha \sum_{i=1}^{h}K_{NPE}^i(G,G') + (1-\alpha) \sum_{i=1}^{h}K_{NPO}^i(G,G')
\end{equation}
where  $\alpha \in (0,1) $ is a tuning parameter, $K_{NPE}^i$ is the NPE kernel and $K_{NPE}^i$ is the NPO kernel defined for $i^{th}$ iteration. The purpose of $\alpha$ is to have a trade off between the two components of NP kernel.

%	\subsection{title}
%	Now we prove with following Lemma that $K_{NPE}$ can be computed from the direct product graph introduced in Section \ref{sec:npr}.

%\subsection{Relation with WL kernels}

\subsection{Other neighborhood preserving kernels}
It can be seen that to ensure a graph kernel to process only on the neighborhood preserving regions, it is enough to take features from the edges that have a common \textit{WL refined edge address} in the graphs. The application of WL refinement iteration on WL-edge kernel \cite{shervashidze2011weisfeiler} results in the process of the neighborhood preserving regions. However, this kernel does not process attribute information. Hence the neighborhood preserving edge (NPE) kernel can be considered as a generalization of WL-edge kernel.

%\textbf{In this aspect, the WL-edge kernel \cite{shervashidze2011weisfeiler} from the WL refinement iteration onwards processes the neighborhood preserving regions. Note that this kernel takes information from the original labels as well which is not done in our formulation. Hence the neighborhood preserving edge (NPE) kernel can be considered as a generalization of WL-edge kernel to the attributed graphs.}

In the case of  WL-subtree kernel and WL-shortest path kernel \cite{shervashidze2011weisfeiler}, the regions processed may include non-neighborhood preserving edges as well. However, WL-shortest path kernel can be made to a neighborhood preserving kernel by imposing an additional constraint such that the shortest paths considered for feature extraction shall contain only neighborhood preserving edges.

%\subsection{Other neighborhood preserving kernels}

It is evident that the walks in the formulated product graph corresponds to neighborhood preserving walks in the argument graphs. Hence random walk kernel \cite{borgwardt2005protein} can be made neighborhood preserving. It is possible to distinguish the shortest paths from  the neighborhood preserving walks. Hence it is possible to make GraphHopper kernel neighborhood preserving and WL-shortest path kernel \cite{shervashidze2011weisfeiler} can be generalized to attributed graphs in the same way as the proposed NPE kernel generalizes the WL edge kernel.

If we insert $d$-edges in the direct product graph definition, as explained in \cite{kriege2012subgraph}, the cliques corresponds to neighborhood preserving subgraph isomorphism and hence subgraph matching kernel \cite{kriege2012subgraph} can be made neighborhood preserving.
\subsection{Recursive computation of NPE kernel from product graph}
We prove with the following theorem that the  neighborhood preserving edge kernel can be computed at each WL iteration  from the product graph defined for the previous iteration which is obtained in a recursive fashion, i.e, the initial product graph formulation alone is enough to find out the product graphs in proceeding WL iterations without explicitly finding them.

\begin{theorem}\label{theorem:3}	
	%Neighborhood preserving kernel at any iteration $:i$ of WL color refinement algorithm can be computed from the product graph defined for iteration $:i$ which is obtained recursively from iteration $:(i-1)$ starting from the initial product graph, $G_P$.
	The product graph at any iteration >1 of the WL color refinement algorithm is the subgraph of the product graph defined for the previous iteration.
	
\end{theorem}
\begin{proof}
	The product graph ($G_{P_i}$) in iteration $:i>1$ can be obtained from the product graph ($G_{P_{i-1}}$) in iteration: $(i-1)$  by deleting certain edges. Suppose $G_{P_i}$ is the product graph obtained by deleting the whole edges of the form $((u,v),(u',v'))$ in $G_{P_{i-1}}$ where $l_{C_{i}}(u) \neq l_{C_{i}}(u') $ and $l_{C_{i}}(v) \neq l_{C_{i}}(v') $ where $l_{C_i}$ is the WL alphabet in iteration $:i$. Our argument is that the edges that can occur in $G_{P_i}$ is already embedded in $G_{P_{i-1}}$ except the deleted edges based on the above rule.
	
	%For h=1, from Lemma \ref{lemma:1}, that is already proved. For iteration $h=i$, the product graph, $G_{P_i}$ is obtained  from product graph,$G_{P_{i-1}}$, at iteration $h=i-1$. $G_{P_i}$ is obtained by deleting the whole edges of the form $((u,v),(u',v'))$ in $G_{P_{i-1}}$ where $l_{C_{i}}(u) \neq l_{C_{i}}(u') $ and $l_{C_{i}}(v) \neq l_{C_{i}}(v') $ where $l_{C_i}$ is the WL alphabet in iteration $:i$. Our argument is that the edges that can occur in $G_{P_i}$ is already embedded in $G_{P_{i-1}}$ or $G_{P_i}$ is the subgraph of $G_{P_{i-1}}$.

	Suppose there exists an edge $((y,y'),(z,z'))$ in $G_{P_i}$ that does not exist in  $G_{P_{i-1}}$. That is label and  neighborhood of $y,y'$ and $z,z'$ with respect to WL labels $l_{C_{i-1}}$ at the iteration $h=i-1$ are identical. If there neighborhood are identical at the iteration $h=i-1$, they would have formed nodes in $G_{P_{i-1}}$ and so the edge between them. Hence the edge $((y,y'),(z,z'))$ in  $G_{P_i}$ is already embedded in $G_{P_{i-1}}$. Hence $G_{P_i}$ formed with the above edge deletion rule is the subgraph of $G_{P_{i-1}}$.
	
\end{proof}

%	With the help of Lemma \ref{lemma:1}, kernel value corresponding to any WL iteration  can be computed from the corresponding product graph at that iteration obtained through the recursion.

\subsection{Computational complexity analysis} \label{complexity}
With efficient sorting techniques, WL color refining for one iteration can be done in $\mathcal{O}(m)$ operations, where $m$ is the number of edges.
NPE kernel defined in equation \ref{Knpe} requires $\mathcal{O}(m^2)$ operations. Considering $d$ be the dimension of attributes, the base kernel computation is of $\mathcal{O}(d)$. Hence NPE kernel is computed in $\mathcal{O}\big( h \times N^2(m^2 \times d)\big)$ where $N$ is the dataset size.

The computational complexity of $K_{NPO}$ is  $\mathcal{O}(N^2 \times  |\Sigma^3|^3)$, where $|\Sigma^3|^3$ is the maximum possible size of $\Lambda$ bounded by $m^2$ for a pair of graphs. This makes the overall computational complexity of NP kernel to be $  \mathcal{O}\big(h \times N^2(m^2 \times d)\big)$.
%The comparison between global and pairwise computations of $K(G,G')$ are done  in Section \ref{comp}.

%To calculate the kernel we have to sum these edges in pairwise manner. Assuming $d$ is the dimension of the attributes, proposed kernel for a pair of graph can be computed in $\mathcal{O}\big( m + \frac{|n^2||n^2 + 1|}{2} + \big(\frac{|n^2||n^2 + 1|}{2}\big)^2 \times d\big) $.

%We can compute the kernel in a global aspect as well. For that we make use of a list of \textit{WL refined edge addresses} that occurs across the dataset and a feature information vector can be formed for each graph listing the edges corresponding to each  \textit{WL refined edge addresses}. Number of distinct pairs of \textit{WL refined edge addresses} that can occur is $\frac{|\Sigma_C||\Sigma_C + 1|}{2}$. To calculate the kernel we have to sum these edges in pairwise manner. Hence proposed kernel can be computed in $\mathcal{O}\big(m + \big(\frac{|\Sigma_C||\Sigma_C + 1|}{2}\big)^2 \times d \big) $ considering a WL iteration for a pair of graphs. Note that when $|\Sigma_C|$ is very large, it may not be effective to compute the kernel in the global aspect. In this case, pairwise computation with recursive definition of product graph is effective.

\begin{algorithm}
	\label{alg1}
	\DontPrintSemicolon
	\SetAlgoLined
	%\KwResult{Write here the result}
	\SetKwInOut{Input}{Input}\SetKwInOut{Output}{Output}
	\Input{Two graphs $G=(V,E)$ and $G=(V',E')'$, number of WL refinement iterations $h$}
	\Output{  $K_{NPE}(G_i,G_j)$ and   $K_{NPO}(G_i,G_j)$}
	\BlankLine
	
	Initialize $K_{NPE}(G_i,G_j)=0$   \;
	\While{$i \leq h$}{
		Do WL color refinement on the graphs \;
		Create direct product graph $G_P = (V_{P},E_{P})$ for $h=i$ avoiding the duplicate edges.\;
		Initialize $\Lambda(key,value)$ as empty dictionary, $k=1$ \;
		Initialize $k_v$ and $\omega$ as empty lists with values zero, $K_{NPE}(G_i,G_j)=0$ \;
		\For{every node $n_i=(u_i,u_i')$ in $|V_P|$}{
			\For{every node $n_j=(v_j,v_j')$ in $|V_P|$}{
				\If{$(u_i,v_j) \in V$ and $(u_i',v_j') \in V'$}{
					Find WL refined edge address, $\epsilon$ \;
					\eIf{$\epsilon \notin \Lambda$}{
						$\Lambda = \Lambda \oplus \epsilon$ \;
						$key(\epsilon) = k$ \;
						$k_v(k) = k_v(k) + \kappa_V(u_i,u_i') \times \kappa_E\big((u_i,v_j),(u_i',v_j')\big)\times \kappa_V(v_i,v_i')$ \;
						$\omega(k)=\omega(k)+1$ \;
					}{
						$k_v(key(\epsilon)) = k_v(key(\epsilon)) + \kappa_V(u_i,u_i') \times \kappa_E\big((u_i,v_j),(u_i',v_j')\big) \times \kappa_V(v_i,v_i')$ \;
						$\omega(key(\epsilon))=\omega(key(\epsilon))+1$ \;
					}
				}

			}		
		}
		
		\For{$i=1$ to length($k_v$)}{
			
			$K_{NPE}(G_i,G_j) = K_{NPE}(G_i,G_j) + k_v(i)/\omega(i)$\;

		}
		
		Initialize $\Lambda(key,value)$ as empty dictionary, $k=1$  and $K_{NPO}(G_i,G_j)=0$ \;
		\For{every graphs $G_i=(V_i,E_i)$ under consideration}{
			\For{every edge $e = (v_p,v_q)$}{
				Find WL refined edge address, $\epsilon$\;
				
				\eIf{$\epsilon \notin \Lambda$}{
					$\Lambda = \Lambda \oplus \epsilon$ \;
					$key(\epsilon) = k$ \;
					%$F(G_i)\{k\} = (v_p,v_q)$, $F(G_i)$ is a tensor which has $|\Lambda|$ elements and in each such element, end vertices of the edges are stored as a list corresponding to the concerned WL refined edge address  $\epsilon$  \;
					$F(G_i)\{k\} = 1$, $F(G_i)$ is a vector (initialized to zero) which has $|\Lambda|$ elements and in each such element, number of occurrence of the concerned WL refined edge address  $\epsilon$  is stored\;
					$k = k+1$\;}
				{
					%$F(G_i)\big\{ key(\epsilon)\big\} = F(G_i)\big\{key(\epsilon)\big\} \oplus (v_p,v_q)$ \;
					$F(G_i)\big\{ key(\epsilon)\big\} = F(G_i)\big\{key(\epsilon)\big\} + 1$ \;
				}

			}
		}
		
		\For{every component $k$ in $\Lambda$}{
			$K_{NPO}(G_i,G_j) = K_{NPO}(G_i,G_j) + minimum(|F(G_i)\{k\}|, |F(G_j)\{k\}|)$ \;		}
		$i=i+1$\;	
	}
	
	\caption{Pairwise computation}
\end{algorithm}
%	\If{$F(G_i)\{k\} \;\bigcap\; F(G_j)\{k\} \neq \phi$}{

\subsection{Algorithms of the pairwise and global kernel computations of NP kernel} \label{compute methods}

We can compute $K_{NPE}$ and hence $K(G,G')$ in two ways.

1. In pairwise manner with product graph.

Although the kernel computation with product graph   using recursion property can have   $n^2$ nodes in the worst case where $n$ is the number of nodes, the number of edges are bounded by $c \times m^2$ where $c$ is a factor that accounts for the edges that result in duplication as explained in Theorem \ref{theorem:2}. The computation steps are detailed in Algorithm \ref{alg1}.

2. In global manner.

We can create a list of \textit{WL refined edge addresses} derived from $\Lambda \subseteq \Sigma^3$ that occurs across the dataset and a pre-computed feature information data can be formed for each graph listing the edges corresponding to each  \textit{WL refined edge address}. Then the kernel can be computed by iterating through this precomputed feature information data. This method by default provides the provision to calculate $K_{NPO}$ because it only requires processing the cardinality of individual refined addresses in the feature information data. The computation steps are detailed in Algorithm \ref{alg2}.

\begin{algorithm}
	\label{alg2}
	\DontPrintSemicolon
	\SetAlgoLined
	%\KwResult{Write here the result}
	\SetKwInOut{Input}{Input}\SetKwInOut{Output}{Output}
	\Input{The graph dataset $D$}
	\Output{The neighborhood preserving edge kernel matrix $K_{NPE}$ and optimal edge assignment kernel matrix $K_{NPO}$}
	\BlankLine
	
	Do WL color refinement algorithm on the graphs\;
	Initialize $\Lambda(key,value)$ as empty dictionary and $k=1$  \;
	\For{every graph, $G_i=(V_i,E_i)$ in $|D|$}{
		\For{every edge $e = (v_p,v_q)$ in $E_i$}{
			%\For{every node $v_k$ in $\mathcal{V}_i$}
			%\If{j$<$k}
			Find WL refined edge address, $\epsilon$ \;
			%\State Find  $L_{\Pi}^{G_i}(v_j,v_k)$ representation
			\eIf{$\epsilon \notin \Lambda$}{
				$\Lambda = \Lambda \oplus \epsilon$ \;
				$key(\epsilon) = k$ \;
				$F(G_i)\{k\} = (v_p,v_q)$, $F(G_i)$ is a tensor which has $|\Lambda|$ elements and in each
				such element, end vertices of the edges are stored as a list corresponding to the concerned WL refined edge address  $\epsilon$  \;

				$k = k+1$ \;}
			{
				$F(G_i)\big\{key(\epsilon)\big\} = F(G_i)\big\{key(\epsilon)\big\} \oplus (v_p,v_q)$ \;
				
				%\EndIf
			}
			
			%\EndFor
		}
		
	}
	
	Initialize $K_{NPE} = 0$ and $K_{NPO} = 0$ \;
	\For{every graph, $G_i=(V_i,E_i)$ in $|D|$ }{
		\For{every graph, $G_j=(V_j,E_j)$ in $|D|$ }{
			\If{$i<=j$}{
				Initialize $k_v=0$ \;
				\For{every element $k$ in $\Lambda$}{
					\If{$F(G_i)\{k\} \;\bigcap\; F(G_j)\{k\} \neq \phi$}{
						\For{every members $e=(v_p,v_q)$ in $F(G_i)\{k\}$}{
							\For{every members $e=(v_p',v_q')$ in $F(G_j)\{k\}$}{
								$k_v = k_v + \kappa_V(v_p,v_p') \times \kappa_E\big((v_p,v_q),(v_p',v_q')\big) \times \kappa_V(v_q,v_q')$ \;

							}
						}
						$K_{NPE}(i,j) = K_{NPE}(i,j) + \frac{k_v}{|F(G_i)\{k\}| \times |F(G_j)\{k\}|}$ \;
						$K_{NPO}(i,j) = K_{NPO}(i,j) + minimum(|F(G_i)\{k\}|, |F(G_j)\{k\}|)$ \;
					}
				}
				
				$K_{NPE}(j,i) = K_{NPE}(i,j)$ \;
				$K_{NPO}(j,i) = K_{NPO}(i,j)$ \;
			}
		}
	}
	\caption{Global computation for one iteration of WL color refinement}
\end{algorithm}

\section{Neighborhood preserving shortest path kernel}\label{sec:npspk}

It is straight forward to extend the concepts of neighborhood preserving edge kernel defined in Section \ref{sec:npek}. One disadvantage with NPE kernel is that the information that gets processed is in terms of edges only, even though the
neighborhood preserving regions are connected. Hence if we process larger subgraphs, the
kernel measure is much more accurate. Shortest paths are a good candidate for analyzing these larger connected regions.

Let $\kappa_V$ be a positive semi definite kernel defined on $V \times V'$. We assume $P,P'$ as the sets that contain shortest paths in graphs $G,G'$ respectively. Then the neighborhood preserving shortest path kernel, $(K_{NPS})$, is defined as,

\begin{equation}
\begin{split}
K_{NPS}(G,G') =  \sum_{\Pi(u_1,u_n) \in P}& \;\;\sum_{\Pi(u_1',u_n') \in P'}  \kappa_V(u_1,u_1') \times \\ &k_\delta \big(\Pi(u_1,u_n), \Pi(u_1',u_n')\big) \times \kappa_V(u_n,u_n')
\end{split}
\end{equation}

$k_\delta$ can be defined as a psd kernel as follows.

\begin{equation} \label{kdelta}
k_\delta\big(\Pi(u_1,u_n), \Pi(u_1',u_n')\big) =\left\{
\begin{array}{ll}
\begin{split}
1,\; if \; & |\Pi(u_1,u_n)|=|\Pi(u_1',u_n')| \; \wedge \\
&   \sum_{i=1}^{n}
l_C(u_i) = l_C(u_i') \; \wedge \\
&\sum_{i=1}^{n-1} l(u_i,u_{i+1}) = l(u_i',u_{i+1}')
\end{split}
\\ 0, \; otherwise  \\
\end{array}
\right.
\end{equation}

In the above definition edges involved in the shortest paths are compared for neighborhood preserving property. Note that the nodes in shortest paths are fixed based on the total order in $\Sigma_C$ and for kernel computation only attributes in source and sink nodes are considered. We can modify this to have kernel value computation of attributes of the nodes in between as well.

The above kernel can be proved as a positive semi-definite kernel in the same way as the NPE kernel. We can formulate the R-convolution relation as a decomposition of the shortest path and the rest of the graph. For each shortest path, we can assign an address like we assign an edge with a \textit{WL refined edge address}. The address can be a string where WL labels of the nodes and edge labels involved in the concerned shortest path are concatenated. With this setting, as in the case of NPE kernel in Section \ref{sec:npek}, we can define an R-convolution kernel corresponding to these addresses.
\subsection{Computational complexity analysis}

The complexity associated with finding shortest paths for a single graph is of $\mathcal{O}(V^2).$ For kernel computation, the shortest paths have to be compared against each other which is of $\mathcal{O}(V^2)$ and along with each comparison, the base kernel has to be computed twice for source and sink nodes. If the attributes are of dimension $d$, this makes the kernel computation $\mathcal{O}(V^2 \times d)$. So the overall computational complexity for a set of $N$ graphs is $\mathcal{O}(N V^2 + h N^2 V^2 d)$.

\section{Experiments} \label{sec:exp}	
The efficiency of the proposed neighborhood preserving kernels was analyzed by subjecting them on real-world data sets and compared its  performance with state-of-the-art graph kernels namely:
Shortest path(SP) kernel \cite{borgwardt2005shortest}, GraphHopper(GH) kernel \cite{feragen2013scalable}, RetGK kernels \cite{zhang2018retgk}, Graph invariant kernel (GIK) \cite{orsini2015graph}, Propogation kernels \cite{neumann2012efficient} and Hash graph kernels \cite{morris2016}.

The components of the proposed NP kernel is analyzed separately, i.e, when $\alpha = 1$ in Equation \ref{main:def} it corresponds to the NPE kernel component that processes the attributes and when $\alpha = 0$ in Equation \ref{main:def} it corresponds to the NPO kernel component that processes the labels. This helps in evaluating the role of attributes and labels separately and also the effectiveness of NP kernel where both information are utilized.

\begin{figure*}[h]
	\centering
	\includegraphics[scale=0.6]{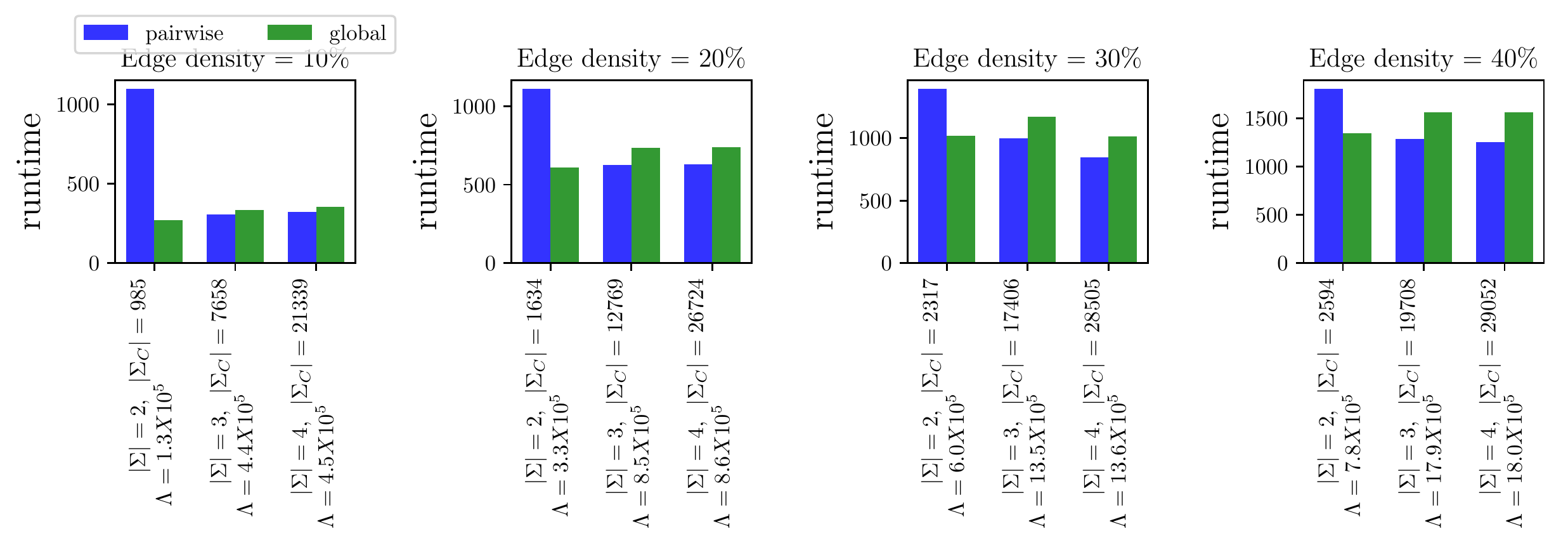}
	%\fbox{\rule[-.5cm]{0cm}{4cm} \rule[-.5cm]{4cm}{0cm}}
	\caption{Runtime comparison of pairwise and global computation of NP kernel for 100 synthetic graphs at graph density = 10\%, 20\%, 30\% and 40\%  respectively. }
	\label{fig3}
\end{figure*}

\begin{table*}
	\caption{Classification accuracy of the proposed kernels with state-of-the-arts.}
	\label{acctable}
	\centering
	\scalebox{0.95}{
		\begin{tabular}{lllllll}
			
			\toprule
			%\multicolumn{2}{c}{Part}                   \\
			%\cmidrule(r){1-2}
			Kernel     & PROTEINS & ENZYMES & BZR & COX2 & DHFR &SYN.new \\
			\midrule
			SP & 73.42 $\pm$ 1.11& \textbf{66.58 $\pm$ 4.06} & 85.76 $\pm$ 2.35 & 79.88 $\pm$ 1.73 & 79.52 $\pm$ 2.40 & 86.72 $\pm$ 3.68\\
			GH & 73.19 $\pm$ 1.79 & 66.33 $\pm$ 2.78 & 82.90 $\pm$ 2.73 & 79.66 $\pm$ 1.17 & 76.65 $\pm$ 3.21  & 88.81 $\pm$ 3.41 \\
			
			RetGK-I & \textbf{75.94 $\pm$ 1.79} & 65.67 $\pm$ 3.07 & 85.58 $\pm$ 2.20 & 78.19 $\pm$ 0.47 & 80.83 $\pm$ 2.10 & 97.08 $\pm$ 1.54\\
			RetGk-II & 74.69 $\pm$ 1.60 & 62.34 $\pm$ 2.94 & 85.76 $\pm$ 1.61 & 78.25 $\pm$ 0.35 & 81.66 $\pm$ 2.41 & 96.94 $\pm$ 1.47\\
			GIK & 72.36 $\pm$ 2.17 & 52.32 $\pm$ 3.89 & 86.31 $\pm$ 2.25 & 79.68 $\pm$ 2.41 & 81.25 $\pm$ 2.56 & 91.73 $\pm$ 2.22 \\
			Prop-diff & 72.60 $\pm$ 2.32 & 38.19 $\pm$ 2.27 & 78.46 $\pm$ 0.59 & 77.94 $\pm$ 0.47 & 72.58 $\pm$ 0.78 & 48.59 $\pm$ 1.17 \\
			Prop-WL & 74.23 $\pm$ 1.80 & 44.85 $\pm$ 1.63 & 78.92 $\pm$ 0.41 & 78.25 $\pm$ 0.35 & 73.41 $\pm$ 0.53 & 46.38 $\pm$ 1.39\\
			HGK-WL & 74.69 $\pm$ 1.98 & 64.30 $\pm$ 3.37 & 81.48 $\pm$ 1.84 & 78.50 $\pm$ 0.57 & 75.47 $\pm$ 2.40 & 81.00 $\pm$ 3.69\\
			HGK-SP & 75.57 $\pm$ 1.89 & 62.60 $\pm$ 3.00 & 82.38 $\pm$ 1.79 & 78.55 $\pm$ 0.65 & 76.61 $\pm$ 2.72 & 96.38 $\pm$ 1.92\\
			\bottomrule
			NPE$_{(\alpha=1)}$ & 71.08 $\pm$ 2.80 & 64.93 $\pm$ 2.97 & 87.13 $\pm$ 1.98 & 82.36 $\pm$ 2.02 &  82.70 $\pm$ 2.35 & 99.51 $\pm$ 0.67\\
			NPO$_{(\alpha=0)}$ & 72.74 $\pm$ 2.02 &  45.67 $\pm$ 3.24 & 88.81 $\pm$ 1.70  & 80.71 $\pm$ 2.92  & 81.07 $\pm$ 2.89   & 97.81 $\pm$ 1.53 \\
			NP & 73.66 $\pm$ 2.55 & 58.10 $\pm$ 3.16 & \textbf{89.12 $\pm$ 2.03} & \textbf{81.16 $\pm$ 2.13} & \textbf{83.98 $\pm$ 2.35} & \textbf{99.74 $\pm$ 0.64}\\
			NPS & 73.87 $\pm$ 2.47  & 58.94 $\pm$ 3.58 &  \textbf{89.24 $\pm$ 2.18} &  \textbf{81.37 $\pm$ 2.26} & \textbf{84.16 $\pm$ 2.43}  & \textbf{99.85 $\pm$ 0.71} \\
			%\midrule
			%Dendrite & Input terminal  & $\sim$100     \\
			%Axon     & Output terminal & $\sim$10      \\
			%Soma     & Cell body       & up to $10^6$  \\
			\bottomrule
	\end{tabular}}
\end{table*}

\subsection{Datasets}
The classification datasets used for analysis were ENZYMES, PROTEINS, BZR, COX2, DHFR and SYNTHETICnew. ENZYMES is a data set of 600 protein tertiary structures obtained from the BRENDA enzyme database \cite{schomburg2004brenda}. PROTEINS is a dataset of chemical compounds with two classes (enzyme and non-enzyme) introduced in \cite{dobson2003distinguishing}.  BZR, COX2, and DHFR \cite{sutherland2003spline} are taken from the Graph data repository \cite{KKMMN2016}. SYNTHETICnew is a synthetic dataset introduced in \cite{feragen2013scalable}. The datasets are all of the binary class except ENZYMES which has six classes.

\subsection{Experimental setup}
The validation process was carried out in the following way. Using the hold-out technique 70\% of the data points were assigned for training and the remaining for testing. 10 fold cross-validation was done on training data for selecting the hyperparameters. A model was then built using the entire training data and its performance was tested on the testing data. The above process  was repeated 30 times and the results reported were averaged over these 30 iterations to nullify the effects of fold assignments.

The classification algorithm used was SVM (with Libsvm implementation \cite{chang2011libsvm}). The penalty parameter $C$ of SVM  was searched in the range [$2^{-7}, 2^{-5}, \dots, 2^{15}$]. The performance parameter used was accuracy. The number of iterations of WL color refinement algorithm was chosen from $\{1,2,3\}$ which is fixed through cross validation. The experiments were done in machine with Intel Xeon i5 2.4 GHz CPU with 80 GB RAM.

We wrote the code for SP and GIK kernel while GH, RetGK kernels, and Hash graph kernel were done with the codes published by authors.  Propagation kernel implementation and (its hyper-parameter selection) was done by the codes published by authors of the GH kernel. The  linear and Gaussian kernel \big($k(x,y) = e^{-\beta\Vert x-y \Vert^2}$\big) where $\beta = 1/d$, ($d$  the dimension of attribute information) were used as the base kernels within the state-of-the-arts. The coding was done in Matlab except for Hash graph kernel which is in Python.

For GIK kernels, WL coloring was taken as  vertex invariant. Two variants of the propagation kernel was implemented. In the 'Prop-diff' variant, the propagation scheme used is diffusion  \cite{neumann2012efficient}. In the 'Prop-WL' variant labels of the nodes are first hashed and the WL propagation scheme is used. Total variance distance was used as the hashing function in both propagation schemes. The bin width
of the hash function was set to $10^{-5}$ and the number of propagation steps for both variants was fixed through cross validation from $\{1,2,\dots 5\}$ . RetGK kernels used also have two variants. RetGK-I is the one with an explicit feature map in RKHS and RetGK-II is the one with approximated mapping. For both approaches, 50 steps of random walks are assumed. For Hash graph kernels,  WL subtree (HGK-WL) and shortest path (HGK-SP) kernels \cite{shervashidze2011weisfeiler} were employed as the base kernels and hashing function used is 2-stable Locality sensitive hashing (LSH) with bin width 1 and
label and hashed attribute information was propagated separately following the practice of authors. Number of WL refinement steps  was fixed through cross validation from $\{1,2,\dots 5\}$.

For the SYNTHETICnew dataset, the node labels were given identical labels discarding the original continuous type values, and  attributes are used as such. For the proposed kernels and GH kernel, the result reported is the best among the case between Linear and Gaussian kernels where they are employed as base kernels. Runtime reported for these kernels is done when Gaussian kernel is employed as the base kernel.

\subsection{Runtime experiments with pairwise and global computation of NP kernel}\label{comp}

This experiment is done to evaluate the pairwise and global computation schemes in calculating the NP kernel as explained in Section \ref{compute methods}.
For this experiment, 4 artificial datasets of 100 graphs with 300 nodes and graph density 10\%, 20\%, 30\%, and 40\% respectively were created. Each dataset has experimented 3 times with the nodes being selecting a random label out of an alphabet $\Sigma$ of size 2,3, and 4 labels respectively, and edge labels are assumed to be identical. The experiment is done for 2 steps of WL color refinement where calculation for pairwise computation is done as per Theorem \ref{theorem:2} and Theorem \ref{theorem:3}. The runtime (in seconds) for both approaches is plotted against the size of WL alphabet $\Sigma_C$ and size of $\Lambda$ for the label size, $|\Sigma| =$2, 3, and 4 respectively. It is assumed that unit time is taken for both approaches for base kernel computations. The variation in graph density is studied since it is the number of edges that affects the computation time as explained in Section \ref{complexity}.

It can be seen from Figure \ref{fig3} that when $|\Sigma_C|$ is small global computation takes less amount of time. In this case,  more nodes are sharing common WL labels and this results in smaller $|\Lambda|$. But the product graph involves lots of nodes and hence larger computation time for pairwise computation scheme. But as $|\Sigma_C|$ becomes larger, pairwise computation time is much better compared to the global. In this case, the nodes sharing common WL labels are relatively less and it will result in a larger $|\Lambda|$. In comparison with the smaller number of nodes in the  product graph, global computation of these larger $\Lambda$ feature information takes more time than pairwise computation. Hence it is actually $|\Lambda|$ that influences more than $|\Sigma_C|$  because as $|\Sigma_C|$ increases, there is no corresponding increase in  $|\Lambda|$ as seen from the experiments especially for cases where $|\Sigma|$=3,4.

In the case of datasets used global computation is used. But from the experiments, it is evident that in the case of datasets with larger $|\Lambda|$ pairwise computation may take lesser time than the global computation.

\subsection{Results and discussion}

%The experiment is done by creating 4 artificial graph dataset of size 100 having  edge density 10\%, 20\%, 30\%, 40\% respectively without considering parallel edges and loops and for two iterations of WL color refinement algorithm. Each dataset is experimented 3 times with the nodes being selecting a random label out of an alphabet of size 2,3, and 4 labels respectively. The runtime (in seconds) for both approach are plotted against the size of WL alphabet $\Sigma_C$ and size of edge features $\Lambda$ as shown in the Figure \ref{fig3}.

%	\subsection{Results and Discussion}

The accuracy with standard deviation obtained are tabulated in Table \ref{acctable} and runtime (wall-clock time) in Table \ref{runtimetable}. We did experiments with the proposed kernels with NPE kernel alone ($\alpha$ = 1) and NPO kernel alone ($\alpha$ = 0) as well as with the formal NP kernel definition, NPS kernel and compared the results with that of state-of-the-arts.   The best results are given in bold letters.

\begin{table}
	\caption{Runtime of the proposed kernels with state-of-the-arts.}
	\label{runtimetable}
	\centering
	\scalebox{0.88}{
		\begin{tabular}{lllllll}
			
			\toprule
			%\multicolumn{2}{c}{Part}                   \\
			%\cmidrule(r){1-2}
			Kernel     & PROTEINS & ENZYMES & BZR & COX2 & DHFR & SYNnew \\
			\midrule
			SP & >5 day& >3 day &  >2 day& >2 day & >3 day & >2 day\\
			GH & 13' 1" &  3' 4" &  58" & 1' 22"  & 3' 7"  & 4' 2"\\
			
			RetGK-I & 2' 57"  & 37" & 17"  & 24"  & 1' 15"  & 30" \\
			RetGk-II & 2.5" & 1" & 0.7" & 0.9" & 1.4" & 13.3"\\
			GIK & 22' 34" & 8' 43" & 7' 16" & 12' 49" & 30' 39" & 35' 11"\\
			Prop-diff & 7"  & 5" & 3.5"  & 3.7"  & 7.5" & 4"\\
			Prop-WL & 12" & 5.3"  & 4" & 5.2" & 9" & 8"\\
			HGK-WL & 3' 42" & 1' 25"  & 1' 2" & 1' 9" & 1' 55" & 2' 11"\\
			HGK-SP & 3' 8" & 1' 6"  & 48" & 52" & 1' 28" & 1' 38"\\
			\bottomrule
			NPE$_{(\alpha=1)}$ & 56' 51"  &  7' 16" & 3' 18"  & 5' 25"  & 11' 37"   & 17' 33"\\
			NPO$_{(\alpha=0)}$ &  6' 37"  &  41"  & 13"   &  19"  &  37"   & 45" \\
			NP &  56' 42" & 7' 10" & 3' 12"  & 5' 20" & 11' 31" & 17' 27" \\
			NPS &   >1 day & 65' 10"  &  23' 37"  & 47' 1"  & 151'   & 198'  \\			
			%\midrule
			%Dendrite & Input terminal  & $\sim$100     \\
			%Axon     & Output terminal & $\sim$10      \\
			%Soma     & Cell body       & up to $10^6$  \\
			\bottomrule
	\end{tabular}}
\end{table}

The NP and NPS kernels have a significant improvement over state-of-the-arts in the case of BZR, COX2, DHFR, and SYNTHETICnew datasets, and in the case of PROTEINS and ENZYMES datasets their performance is reasonably good. Note that NPE kernel which processes attribute information and NPO kernel which processes only labels where $\alpha$ tuning is not required performs better than state-of-the-arts in BZR, COX2, and SYNTHETICnew datasets. In the case of the DFHR dataset, the NPE kernel outperforms the state of the arts whereas the performance of  the NPO kernel is on par with them.  For datasets except for ENZYMES and COX2, NP kernel augmented with NPO kernel performs significantly better than NPE kernel. This validates our argument about the need for processing the label and attribute information independently. This also gives evidence to the fact that in the case of graph data analysis, if attributes are present it cannot be neglected. This is important since most of the kernels developed in this regard could only process labels. In comparison with the performance of the proposed kernels with the discretization algorithms (Propagation and HGK kernels), the proposed designs perform better although the runtime of discretization based approaches are low.

The runtime of NP and NPS kernels are reported for the global computation approach with Gaussian kernel being the base kernel and it is reasonable compared with state-of-the-arts. Considering the runtime, propagation kernels are the fastest. But they use discretized attribute information and hence their performance is lower than state-of-the arts. Although the runtime of RetGK kernels is better, NP kernel processes the label and attribute information independently which makes the performance of those better. Since NP and NPS kernel establishes a well-defined correspondence between subgraphs, its performance is better than GraphHopper kernel despite the difference in runtime. In comparison with the shortest path kernel and Graph invariant kernels, the runtime for NP kernels is better. The performance of the GIK kernel are close with that of the NP kernel while the performance of NP kernels compared to the shortest path kernel is better in the datasets, an exception being ENZYMES.

Note that the runtime of the NPS kernel can be improvised with the efficient computation approaches like the strategies introduced in \cite{feragen2013scalable}. The difference in the runtime of NPE and NP kernels are negligible. The reason is because of the global computation scheme. When we calculate the NPE kernel, the steps involved by default provides a way to calculate the NPO kernel as well, as described in line: 29 of Algorithm \ref{alg2}. But once we are only concerned with the computation of NPO kernel, the runtime is significantly reduced since the processing of attributes is not required.

\section{Conclusion}\label{sec:con}

We designed kernels based on the neighborhood preserving property where the similarity between argument graphs is well defined and visualized. The neighborhood preserving concept helps to define a well-defined correspondence between subgraphs computed during kernel computations. The kernels can be recursively computed from the product graph that helps in an efficient computation procedure for large alphabets occurring in the WL color refinement algorithm. The  proposed kernel which independently processes the attribute and label information is found to be very effective in the graph classification tasks. Their performance is found to be superior in comparison with the state-of-the-arts.

\bibliographystyle{ieeetr}

\end{document}